\def\year{2015}
\documentclass[letterpaper]{article}
\usepackage{aaai}
\usepackage{times}
\usepackage{helvet}
\usepackage{courier}
\usepackage{url}
\usepackage{amsmath}
\usepackage{amsthm}
\usepackage{subcaption}
\usepackage{graphicx}
\DeclareMathOperator*{\argmax}{arg\,max}
\begin{document}
%
\newtheorem{theorem}{Theorem}
\title{Canonical Autocorrelation Analysis}
\author{Maria De-Arteaga, Artur Dubrawski, Peter Huggins\\
Carnegie Mellon University\\
Pittsburgh, PA 15213\\
mdeartea@andrew.cmu.edu, awd@cs.cmu.edu, phuggins@andrew.cmu.edu\\
}
\nocopyright
\maketitle
\begin{abstract}
\begin{quote}
We present an extension of sparse Canonical Correlation Analysis (CCA) designed for finding multiple-to-multiple linear correlations within a single set of variables. Unlike CCA, which finds correlations between two sets of data where the rows are matched exactly but the columns represent separate sets of variables, the method proposed here, Canonical Autocorrelation Analysis (CAA), finds multivariate correlations within just one set of variables. This can be useful when we look for hidden parsimonious structures in data, each involving only a small subset of all features. In addition, the discovered correlations are highly interpretable as they are formed by pairs of sparse linear combinations of the original features. We show how CAA can be of use as a tool for  anomaly detection when the expected structure of correlations is not followed by anomalous data. We illustrate the utility of CAA in two application domains where single-class and unsupervised learning of correlation structures are particularly relevant: breast cancer diagnosis and radiation threat detection. When applied to the Wisconsin Breast Cancer data, single-class CAA is competitive with supervised methods used in literature. On the radiation threat detection task, unsupervised CAA performs significantly better than an unsupervised alternative prevalent in the domain, while providing valuable additional insights for threat analysis.
\end{quote}
\end{abstract}
\section{Introduction}
\label{introduction}

Canonical Correlation Analysis (CCA) is a useful tool for finding multivariate correlations between two sets of features. It works well when the features describing a problem or task at hand are naturally divided into separate sets. For example, it can be used to find correlations between genes and diseases, when both types of information are available for a set of patients. We believe that the capability of identifying relationships between sets of features is of general interest, even in cases when the natural or intuitive splits of features into separate subsets are not be obvious.

In this work, we present Canonical Autocorrelation Analysis (CAA). Here, the word \textit{autocorrelation} reflects correlations existing within a single set of features. CAA automatically discovers how to separate features into ``inputs" and ``outputs" by selecting these subsets that maximize their mutual correlation. Just like CCA, CAA can identify multiple pairs of subsets of features if the corresponding correlations exist. The capability of automatically identifying multivariate structures of correlation make CAA naturally fitting in applications suitable for single-class learning, unsupervised learning or data-driven discovery. 

CAA and the well known Sparse Principal Component Analysis (Sparse PCA) are fundamentally different. While CAA resembles Sparse PCA in the sense that they both find sparse representations of data contained in one matrix, PCA finds one-dimensional projections of data that maximize \textit{variance} of data in each projection, while CAA finds two-dimensional projections where \textit{correlation} is maximized. And even though a two dimensional projection can be obtained by combining a pair of principal components from Sparse PCA, such projections are specifically designed to be mutually uncorrelated. CAA specifically seeks projections composed by pairs of strongly correlated components, enabling discovery of hidden characteristic correlations in data, which cannot be found with other methods.

An example where CAA is useful is the task of radiation threat detection. In this problem we consider gamma-ray spectra collected with a portable sensor moving in an urban environment. Such data is typically represented as a vector of photon counts registered by the sensor at subsequent discrete and disjoint intervals of energy. These vectors, called in the application domain the energy spectra, become data points of our analysis. The fundamental task involving such data is to characterize variability of benign background radiation across multiple intervals of energy. Then, when a threatening source of radiation is present in the scene, the additional photon counts and their characteristic distribution may be identified in the noisy data as an anomaly. 

The simplest methods of radiation threat detection rely on counting total numbers of photons registered in a unit of time without regard to their specific energies. However, this approach is empirically inferior to the more detailed analysis which also considers correlations between counts observed in distinct energy bins, especially if the ambient radiation is highly variable. An example of such an environment is a cluttered urban scene, where total background counts often vary by a factor of 2 or 3. This magnitude of variation obfuscates visibility of threatening sources and makes threat detection challenging. Harmless objects like brick structures and cat litter, for example, emit high levels of radiation, but it is clearly not desirable that alarms go off every time the detector encounters a brick building. Therefore, it is necessary to do a more detailed analysis of the gamma-ray spectra, as often what identifies threats from harmless objects is not the total amount of energy, but the patterns photon counts follow in particular energy bins. Radiation threat detection is also challenging because different types of threats follow different spectral patterns, and even if templates of some common threat types are available, relying on supervised analysis of field data is risky. Supervised detectors may fail to detect threats that were not present in the training data, or which were shielded in a particularly unexpected fashion. Therefore, efforts have been made to develop unsupervised methods that successfully detect threats without relying on source templates. 

Applying the method presented in this paper, it enables us to characterize harmless radiation with a structure of correlations spanning sets of energy bins. Once this characterization is established, it can be used for spectral anomaly detection, as threats can be expected to not follow the same characterization as harmless ambience. We show in experiments that the ability of CAA to identify parsimonious subsets of features and later use it to model background radiation variability makes it more robust at threat detection than one of the most popular unsupervised methods used in the domain: a Principal Component Analysis (PCA)  based spectral anomaly detection method that considers all dimensions of spectra in linear combinations corresponding to subsequent principal components.

Breast cancer detection presents a similar challenge. Even though the rich body of previous research on the Breast Cancer Wisconsin data set~\cite{WoM1990} focuses on supervised learning \cite{Qui1996}\cite{NDK1999}\cite{AJS2003}\cite{PoG2007}\cite{SMD2013}\cite{ZYL2014}, there is reason to believe this problem should be approached as an anomaly detection problem. There exist many diverse variants of cancer, new types are still being discovered, and some of the known types are very rare, making it hard to train reliable detectors of these types of cancer.  CAA-based anomaly detection, however, shows to be a strong contender with many supervised (and therefore more informed about frequent types of cancer) methods, while potentially being more reliable than supervised methods at detecting types of cancer that were not well represented in the training set.

Note that sometimes we refer to the method as unsupervised and other times we refer to it as single-class. This is because in the radiation domain it takes as input unlabeled ambience data, while in the breast cancer domain it is trained on cases labeled as benign. Beyond this, however, the method is exactly the same in both cases.

In the remainder of this paper, Section~\ref{background} presents a brief review of related work, in Section \ref{sec:method} the proposed methods are described in detail,  and Section \ref{sec:results} contains results of experiments applying CAA to three different datasets: synthetic data, the Breast Cancer Wisconsin benchmark data set and the radiation threat detection data set. Section \ref{sec:conclusion} concludes our paper.

\section{Background}\label{background}

Canonical Correlation Analysis is a statistical method first introduced by ~\cite{Hot1936}, useful for exploring relationships between two sets of variables. It is used in machine learning, with applications to multiple domains; previous applications to medicine, biology and finance include  \cite{FBL2002}, \cite{DVV2006}, \cite{TeH2012} and \cite{WiT2009}.

A modified version of the algorithm was proposed by \cite{WiT2009}\cite{WTH2009}. Sparse CCA, an $L_1$ variant of the original CCA, adds constraints to guarantee sparse solutions. This limits the number of features being correlated. Their formulation of the maximization problem also differs from the traditional CCA algorithm. We will use this version since it is suitable for our needs.

Principal Component Analysis (PCA) and its variant Sparse PCA are also related to our research. PCA and Sparse PCA find orthogonal projections of the data onto linear spaces where the variance of the data is maximized. This procedure aims at preserving as much variance as possible, while reducing dimensionality. The goal of CAA is different: it finds two-dimensional projections that emphasize correlation, where such correlations constitute interpretable patterns that are present in the data. 

The work that most resembles our research is \cite{FBL2002} and \cite{DVV2006}. Using the notion of autocorrelation, they attempt to find underlying components of functional magnetic resonance imaging (fMRI) and electroencephalogram (EEG), respectively, that have maximum autocorrelation, and to do so they use CCA. The type of data they work with differs from ours in that their features are ordered, both temporally and spatially. To find autocorrelations, they take $X$ as the original data matrix and construct $Y$ as a translated version of $X$, such that $Y_t=X_{t+1}$. Since our data is not ordered we cannot follow the same procedure and must instead develop a new technique of finding autocorrelations.

In the radiation analysis domain, PCA-based anomaly detection is commonly used \cite{Tan2015}. PCA-based spectral anomaly detection works by essentially calculating the magnitude of the residual after a background-subtracting ``projection", where the projection is either a strict projection onto the subspace spanned by the top few principal components of the covariance matrix, or a dilation modified projection where the correlation (not covariance) matrix is used to learn the low dimensional projection and then appropriate scaling of the measurement dimensions is performed before projection and scaled back after projection. In any event, the ``projection" transformation computes the estimated background contribution to a radiation measurement, assuming that the top few principal components represent expected ``typical" background variation. After projection, the magnitude of the residual essentially provides the PCA-based spectral anomaly score. In the case of CAA, we build an anomaly detection method by letting multiple-to-multiple combinations of subsets of energy bins that are well correlated provide a characterization for normal data, which in the case of radiation corresponds to background. Such characterizations are used as the basis of the anomaly detection method described in \ref{subsec:anomaly}, which identifies anomalies when they depart from the expected patterns of multiple-to-multiple feature relationships.




\section{Canonical Autocorrelation Analysis}
\label{sec:method}


Given two matrices $X$ and $Y$, CCA aims to find linear combinations of their columns that maximize the correlation between them. Usually, $X$ and $Y$ are two matrix representations for one set of data points, so that each matrix is using a different set of variables to describe the same data points. 

We use the formulation of CCA given by \cite{WiT2009}. Assuming $X$ and $Y$ have been standardized and centered, the constrained optimization problem is:
\begin{equation}
\label{eq:CCA}
\begin{split}
max_{u,v} u^TX^TYv \\ 
||u||_2^2 \leq 1, ||v||_2^2 \leq 1 \hspace{.2in} ||u||_1 \leq c_1, ||v||_1 \leq c_2
\end{split}
\end{equation} 

When $c_1$ and $c_2$ are small, solutions will be sparse and thus only a few number of features are correlated.

Our goal is to find correlations within the same set of variables. Therefore, our matrices $X$ and $Y$ are identical. Applying Sparse CCA when $X=Y$ results in solutions $u=v$, corresponding to sparse PCA solutions of $X$~\cite{WTH2009}.

The naive approach for finding correlations between disjoint subsets of features would consist of trying multiple ways of splitting the features into two groups and applying Sparse CCA. This is computationally infeasible, as the number of times Sparse CCA would have to be applied is in $O(\binom {m} {s})$, where $m$ is the number of columns of $X$ and $s=|\{u_i \neq 0\}|$ is determined by the constraints $c_1$ and $c_2$ that would be applied to the original matrix $X$ to obtain the desired level of sparseness.

We develop a modified version of the algorithm capable of finding such autocorrelations by imposing an additional constraint on Equation \ref{eq:CCA} to prevent the model from correlating each variable with itself. Using the Lagrangian form, the problem can be written as follows:

\begin{equation}
\begin{split}
max_{u,v} u^TX^TXv-\lambda u^Tv \\ 
||u||_2^2 \leq 1, ||v||_2^2 \leq 1 \hspace{.2in} ||u||_1 \leq c_1, ||v||_1 \leq c_2
\end{split}
\label{CAA_original}
\end{equation} 

This will penalize vectors $u$ and $v$ for having high values for the same component, which is precisely what we are trying to avoid. Through proper factorization we are able to solve this through Sparse CCA.

\begin{theorem}
Solving Equation \ref{CAA_original} is equivalent to solving Sparse CCA for the pair of matrices 
\begin{center}
$\hat{X}=[V(S^2-\lambda I)]^T $ and $ \hat{Y}=V^T$.
\end{center}

\noindent where the Singular Value Decomposition of $X$ is $X=USV^T$ and $I$ is the identity matrix. 
\end{theorem}

\begin{proof}
First, notice that 

\begin{center}
$u^TX^TXv-\lambda u^Tv=u^T(X^TX-\lambda I)v$.
\end{center}

Therefore, the CAA problem can be written as:

\begin{equation}
\begin{split}
max_{u,v} u^T(X^TX-\lambda I)v \\ 
||u||_2^2 \leq 1, ||v||_2^2 \leq 1 \hspace{.2in} ||u||_1 \leq c_1, ||v||_1 \leq c_2
\end{split}
\label{CAAeq}
\end{equation} 

Finding the Singular Value Decomposition of $X$, we have:

\begin{center}
$\begin{array}{rcccl}
&X&=&USV^T   \\ 
\Rightarrow & X^TX & = & VS^2V^T\\
\Rightarrow & X^TX-\lambda I & = & VS^2V^T-\lambda V V^T\\
& &= & V(S^2-\lambda I)V^T
\end{array}$
\end{center}

Now, setting $\hat{X}=[V(S^2-\lambda I)]^T $ and $ \hat{Y}=V^T$, the problem becomes:

\begin{equation}
\begin{split}
max_{u,v} u^T\hat{X}^T\hat{Y}v \\ 
||u||_2^2 \leq 1, ||v||_2^2 \leq 1 \hspace{.2in} ||u||_1 \leq c_1, ||v||_1 \leq c_2
\end{split}
\end{equation} 

This problem can be solved using Sparse CCA, and since the solutions obtained with this method are independent of the factorization of $\hat{X}^T\hat{Y}$, solving this is equivalent to solving the CAA maximization problem in Equation \ref{CAAeq}.
\end{proof}

The maximization problem is not convex with regard to $\lambda$ and is sensitive to variations of this parameter. Our approach to solving CAA performs a grid search and returns up to $m$ pairs of canonical vectors $u,v$ (it may return fewer if there are not enough vectors for which the data has a linear correlation when projected onto the space defined by $u$ and $v$). 

The grid search uses an evaluation metric for relative sparseness developed for this purpose. Vectors $u$ and $v$ are considered to have relative sparseness if components with high values in $u$ do not coincide with high valued components in $v$. This can be measured by

\begin{equation}
t(u,v)=1-\sum_i|u_iv_i|
\end{equation}

The case of $t(u,v)=1$ corresponds to vectors $u$ and $v$ having disjoint support. The $i$th CAA solution is obtained by finding the minimum value of $\lambda$ for which the corresponding CCA solution of $\hat{X}$ and $\hat{Y}$ have a relative sparseness of 1. Now, instead of applying Sparse CCA O($\binom {m} {s}$) times, as in the naive approach, we apply it O($\lambda$) times, where O($\lambda$) refers to the number of points in the grid search.

\subsection{CAA-based anomaly detection}
\label{subsec:anomaly}

How can the outcome of CAA be used once it has been applied to a matrix $X$? CAA produces several multiple-to-multiple linear correlation patterns. If the only goal is data characterization, one can analyze the coefficients in the canonical projections to understand which items are relevant for a certain data set. In addition, such projections can be used as the basis of an anomaly detection method, which we introduce in this section.

CAA can be applied to a data set of data points that are assumed to not be anomalous. Each CAA solution, a pair of canonical vectors $(u,v)$, maps every data point into a new bi-dimensional space, where the x-axis corresponds to $u^tX^t$ and the y-axis corresponds to $Xv$. We expect top canonical projections to yield scatter plots in which the training data shows a strong diagonal tendency. To characterize a shape of such distributions we can use some density model, e.g. bivariate Gaussian. After doing this for each canonical solution, the data set is characterized by:

\begin{itemize}

\item $\{(u_i,v_i): i=1,...,k$ for some $k \leq n\}  $

\item $\{(\mu _i,\sigma _i): i=1,...,k$ for some $k \leq n\}  $
\end{itemize}

For each $i$ we obtain a characterization of the training data that involves multiple features. Given a new data point, it can be mapped onto all $k$ canonical spaces, and the Mahalanobis distance to the Gaussian on each of these spaces can be computed. If the new point can be characterized in the same way as the training data, all of the Mahalanobis distances should be small. It can be expected that an anomalous point cannot be characterized the same way as a normal one; it will likely fail to follow one or multiple of these characterizations, which will result in one or multiple large Mahalanobis distances. It may be convenient to marginalize the resulting distribution of scores. One conceivable option for this is maximization, as defined in Equation \ref{maxMahScore}. 

\begin{equation}
\begin{split}
s(x) = \max_{i=1,...,k}D_{M_i}((u_i^tx^t,xv_i)) 
\end{split}
\label{maxMahScore}
\end{equation} 

where $D_{M_i}(\cdot)$ is the Mahalanobis distance to $N(\mu_i, \sigma_i)$ and $x$ is the current observation. 

The anomaly detection method can be naturally generalized to fit any distribution to the projection of the data, replacing the Mahalanobis distance with the likelihood or a similar metric. Additionally, $s(x)$ is taken to be the maximum distance because in the present applications a point is considered anomalous when it fails to follow any of the characterizations. In the case of radiation, for example, it is known by the experts that threats may only manifest in a small subset of the energy bins, while for the rest its behavior resembles that of harmless objects. However, there might exist applications where data is only considered anomalous when it fails to follow all or almost all of the characterizations, in which case $s(x)$ can be replaced for a cumulative or a robust metric over the $k$ scores.

\section{Experiments}
\label{sec:results}

\subsection{Synthetic data}

The first set of experiments aims to assess and illustrate how known imposed correlations can be successfully retrieved. 

A Gaussian bivariate distribution is generated given a mean and a covariance, and 200 points are sampled from it. A matrix $X$ of dimensions $200 \times 20$ is created such that there exist sparse vectors $u,v$, each with two non-zero components, for which $(u^tX,Xv)$ correspond to the previously generated Gaussian. $70\%$ of data is used to train a CAA model and the rest is used for testing.  Figure \ref{Gaussian_synthetic} (left) contains a plot of the points sampled from the Gaussian, where the axes indicate the linear combinations of columns of X that map the original data onto the Gaussian, and Figure \ref{Gaussian_synthetic} (right) shows the projection of both training and testing data onto the space determined by the first pair of canonical vectors retrieved by CAA, where the equations on the axes correspond to the correlation they establish. Note that the method is able to successfully identify the four columns for which multiple-to-multiple linear correlation exists.

\begin{figure}
\centering
  \begin{subfigure}[b]{0.22\textwidth}
    \includegraphics[width=\textwidth]{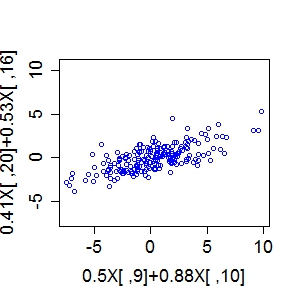}
  \end{subfigure}
  \begin{subfigure}[b]{0.22\textwidth}
    \includegraphics[width=\textwidth]{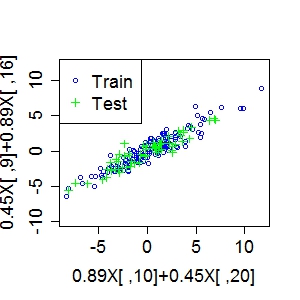}
  \end{subfigure}
  \caption{Comparison between synthetic correlation pattern (left) and correlation pattern retrieved by CAA (right). Equations have the form $k_iX[,i]+k_jX[,j]$, where $X[,i]$, $X[,j]$ are the $i$th and $j$th columns of X and $k_i$, $k_j$ are the linear combination coefficients.}
  \label{Gaussian_synthetic}
\end{figure}

\subsection{Breast cancer detection}

The Wisconsin Breast Cancer data set contains 699 cases of tumors, 683 after removing those with missing values, labeled as either malignant (239 cases) or benign (444 cases). Each sample corresponds to a tumor described by 9 numeric variables that take on values between 1 and 10: clump thickness, uniformity of cell size, uniformity of cell shape, marginal adhesion, single epithelial cell size, bare nuclei, bland chromatin, normal nucleoli and mitoses. Previous work has been done with this data to build predictive models that distinguish malignant samples from benign samples \cite{Qui1996}\cite{NDK1999}\cite{AJS2003}\cite{PoG2007}\cite{SMD2013}\cite{ZYL2014}. These approaches use supervised classifiers. The use of single-class approaches to detect malicious tumors may however be beneficial, given that some types of cancer are fairly rare and different from more typical cases. There has been research showing that cancer can follow multiple developmental pathways in asynchronous orders~\cite{AKC2014}, and even these studies are not known to be complete, so there is much motivation to anticipate at least some cancerous tumors that do not follow characteristics of the most common types of cancer, but are still identifiable as deviating from typical patterns in benign data in some way.

We approach this as an anomaly detection problem, and train CAA using benign cases. We then use the CAA-based anomaly detection method to score both anomalous and normal test cases. As a comparison, we also apply the PCA-based anomaly detection approach previously described in Section \ref{background}. 

The Area Under the ROC Curve (AUC) for the CAA-based approach is $0.979$ within the confidence interval $[0.934,1]$, while PCA-based approach obtains an AUC of $0.720$ within the confidence interval $[0.612,0.824]$. 

Most results in the literature are only reported in terms of accuracy, without specifying the rates of false positives and false negatives that such levels of accuracy entail. To make our results comparable to previous literature, we do a basic grid search for a score threshold on the training data using accuracy as criterion. Table \ref{table_breastcancer} compares results in literature, CAA-based results and PCA-based results using 10-fold cross validation, and its standard deviation. Performance of CAA is comparable to that of supervised methods and significantly better than that the single-class alternative (PCA). Additionally, CAA can be expected to be more reliable than a supervised method when tasked to detect types of cancer that were not present in the training set. The fact that it is fully interpretable is also an advantage, as domain experts, in this case doctors, may be reluctant to accept decision support methods that act as black-boxes, such as support vector machines or neural networks, plus classification criteria that are easily interpretable by a domain expert may lead to new qualitative insights of how different types of cancer behave.

\begin{table*}[ht]
\centering 
\begin{tabular}{c c c l} 
\hline
Type & Method & Reference & Accuracy (\%)  \\
\hline
Supervised & K-SVM & \cite{ZYL2014} & 97.38 \\
Supervised & LLWNN & \cite{SMD2013} & 97.2  \\
Supervised & LS-SVM & \cite{PoG2007} & 98.53   \\
Supervised & Supervised fuzzy clustering & \cite{AJS2003} & 95.57 \\
Supervised & NEFCLASS & \cite{NDK1999} & 95.06 \\
Supervised &  C4.5  & \cite{Qui1996} & 94.74   \\
Single-class & \bf{PCA}& & 70.11 $\pm$ 5.76    \\
Single-class & \bf{CAA}& & 94.72 $\pm$ 3.27    \\

\hline
\end{tabular}
\caption{Comparison of classification accuracy on the Wisconsin Breast Cancer Data Set for 10-fold cross-validation. Standard deviation is not available in the literature, hence it is only included for CAA.}
\label{table_breastcancer}
\end{table*}




\subsection{Radiation threat detection}

The radiation data used in our experiments is featurized into 128 different energy bins, and reflects photon counts obtained from gamma-ray spectrometer measurements. There are 20,000 records available for harmless background data, and 10,000 records for each of 15 types of injected data, which simulates different threats. 

As it was previously explained, PCA-based spectral anomaly detection assumes that the top few principal components represent expected background variation, and uses the residual after removing these top components to provide a spectral anomaly score. In the case of CAA, multiple-to-multiple combinations of bins that are well correlated provide a characterization model for background radiation. This model can be used as the basis of the anomaly detection method described in Section \ref{subsec:anomaly}, which identifies threats when radiation spectra depart from the expected patterns.

\begin{figure}[t]
   \includegraphics[width=0.95\linewidth]{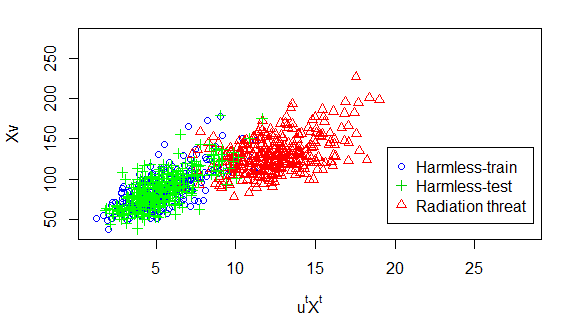}
\caption{Projection of background radiation and threat radiation onto space determined by CAA canonical projections.}
\label{proj_radiation}
\label{fig:radiation_proj}
\end{figure}

\begin{figure}[t]
\centering
\includegraphics[width=0.9\linewidth]{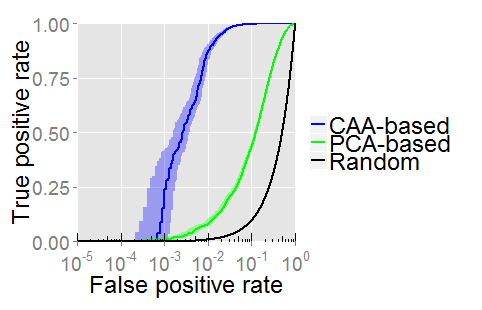}
  \caption{ROC curve comparing CAA-based and PCA-based anomaly detection methods for analyzed threat. X-axis is in log-scale.}
  \label{ROC:radiation}
\end{figure}%

Figure \ref{proj_radiation} shows an example of the mapping onto the space determined by CAA canonical vectors ($u,v$) for one particular threat. This particular pair $u,v$ corresponds to the one that is used most often to generate the scores for the data points belonging to that threat, meaning the one where the maximum Mahalanobis distance to the Gaussian characterizing background radiation is found most often, as defined in Equation \ref{maxMah}.

\begin{center}
\begin{equation}
(u,v) = \argmax_{u_i,v_i}D_{M_i}((u_i^tx^t,xv_i))|_{i=1}^k 
\end{equation}
\label{maxMah}
\end{center}

As Figure \ref{proj_radiation} shows, in this example the threat-injected data distribution visibly diverges from the test set distribution of benign data. For this threat type, CAA model achieves an AUC of 0.995, while PCA-based detector has an AUC of 0.821. The ROC plots are shown in Figure~\ref{ROC:radiation}, with the false positive rate axis in logarithmic scale, to enhance view at low false positive rates. 

Both models were trained using 10,000 background records, and applied to detect different threats in data sets where the 10,000 samples of injected data corresponding to that threat are combined with the remaining 10,000 background records. PCA-based technique outperforms CAA-based anomaly detection in only one case among 15, for which the AUC's are 0.963 and 0.831, correspondingly. Figure~\ref{AUC} summarizes the results for all 15 threats. A Student t-test yields a p-value of 0.00587, indicating that CAA performs significantly better than PCA.

\begin{figure}
 \centering
 \includegraphics[width=1\linewidth]{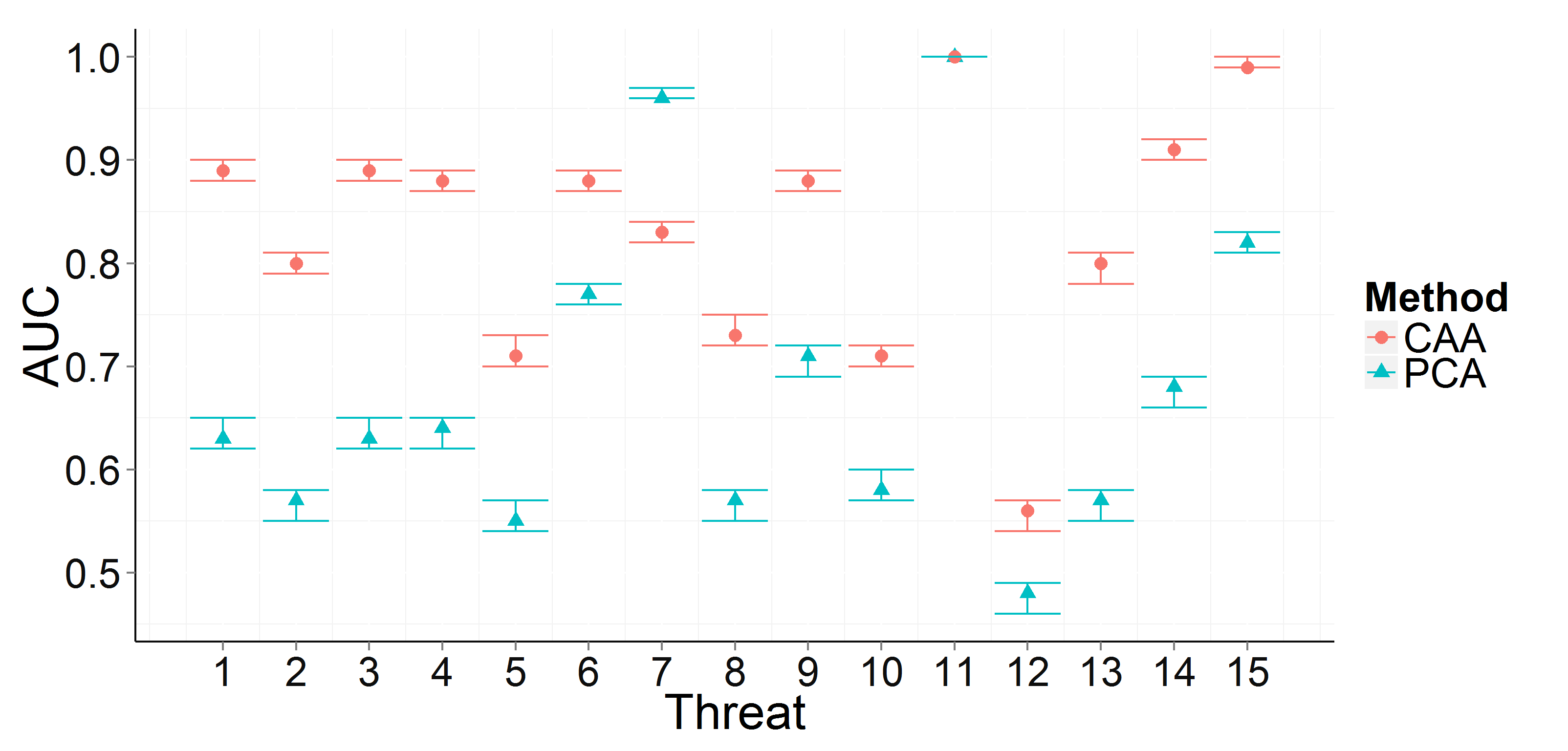}
 \caption{AUC and confidence intervals for PCA and CAA anomaly detection applied to radiation threats. A Student’s t-test applied to this results give a p-value of 0.00587, showing CAA performs significantly better than PCA.}
 \label{AUC}
 \end{figure}

In addition, the proposed method is readily interpretable. When a spectral measurement is identified as a possible threat, the bins on which it fails to follow background patterns are pointed out. This has two main advantages: first, when analyzing an individual data point the user knows what energy bins the algorithm used to make its decision, for easy adjudication of the results. Secondly, when applied to a batch of data associated to a particular threat type, it is possible to identify the bins on which the threat's behavior differs from the background behavior, providing the way to characterize the threat. 

Figure \ref{sunset} shows the frequency with which energy bins are used to identify anomalies. On the top, it shows the aggregated counts for a threat batch associated to one particular threat type; on the bottom, it shows an example of the radiation spectrum for a spectrum the method correctly labels as threat, the average of the background radiation spectra used for training, and colours those that were used to label the data as anomalous. It is interesting to see that even though the method is fully unsupervised, such bins correspond to spikes in the injected threat template.

\begin{figure}[ht]
\centering
\includegraphics[width=0.93\linewidth]{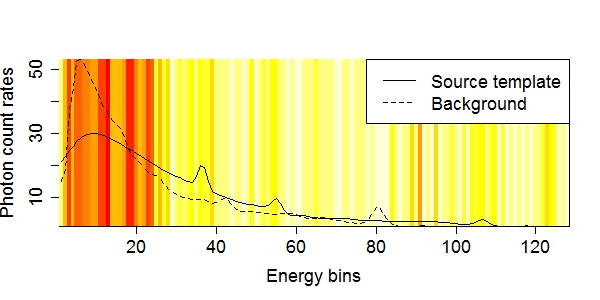}
 \includegraphics[width=0.93\linewidth]{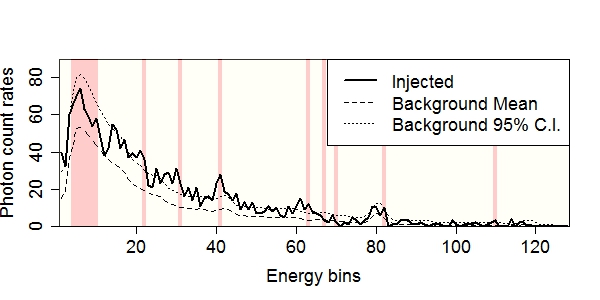}
\caption{Heat graphs indicating which energy bins are used to label elements as threats. Top: Cumulative use for spectra in threat batch. Bottom: Individual spectrum.}
\label{sunset}
\end{figure}



\section{Conclusion}
\label{sec:conclusion}

We have presented Canonical Autocorrelation Analysis (CAA), a new method that automatically finds subsets of features of data that form strong multiple-to-multiple linear correlations. We have also shown how CAA can be useful at anomaly detection tasks which involve multivariate numeric data where detecting changes in the structure of their mutual correlations may be of interest. We have applied CAA-based anomaly detector to breast cancer detection and radiation threat detection, and it was successful in both these applications. We believe CAA can be a valuable tool in many scenarios where in spite of an existing division between ``non-anomalous" and ``anomalous", the ``anomalous" class is composed of multiple sub-classes, many of which might be unknown or underrepresented in the training data, but are still important to detect. The examples we have shown confirm that: it is crucial to be equally capable of detecting known radiation threats and rare ones, and it is equally important to detect a common type of cancer and an unusual one. 

The models built by CAA are readily interpretable. The canonical projections are sparse, and even though the anomaly detection method analyzes every projection found, often it is sufficient to consider only one of the pairs to adjudicate a query data as possibly anomalous, and each test case may invoke a different canonical pair as the most useful for handling it. Interestingly, these most useful projections are usually also the ones which provide the most interpretable intuition on why and how the particular query does not seem to fit the reference distribution. Consider that each canonical projection spans the axes of a 2-dimensional Cartesian system which happen to be sparse and easy to interpret linear combinations of the native features of data, and it should be clear why CAA may be particularly useful in applications where machine learning is used to aid humans in their decision making.

The next step in our research is to extend CAA towards supervised learning and to enable discovery of non-linear relationships via kernelization. 


\bibliography{bib}
\bibliographystyle{aaai}

\end{document}